\documentclass[11pt,a4paper]{article}
\usepackage[hyperref]{emnlp2020}
\usepackage{times}
\usepackage{latexsym}

\usepackage{microtype}
\usepackage{graphicx}
\usepackage{subfigure}
\usepackage{booktabs} 
\usepackage{amsmath}
\usepackage{amssymb}
\usepackage{amsthm}
\usepackage{dsfont}
\usepackage{multirow}
\usepackage{bbding}
\usepackage{caption}
\usepackage{tabularx}
\usepackage{pifont}
\usepackage{hyperref}

\usepackage{cuted}

\newcommand{\eos}{\ensuremath{\left<\text{eos}\right>}}
\newcommand{\bos}{\ensuremath{\left<\text{bos}\right>}}

\usepackage{microtype}

\aclfinalcopy 

\title{Consistency of a Recurrent Language Model With Respect to \\ Incomplete Decoding}

\author{Sean Welleck$^1$\thanks{~~Equal contribution. Correspondence to: Sean Welleck {\tt wellecks@nyu.edu}.}~~~~~~~~~~~~~Ilia Kulikov$^1$\footnotemark[1]~~~~~~~~~~~~~Jaedeok Kim$^2$\thanks{~~Work done at New York University.} \\
\bf Richard Yuanzhe Pang$^1$~~~~~~~~~~~~~Kyunghyun Cho$^1$$^,$$^3$ \\
$^1$ New York University \ \ \ \ $^2$ Samsung Research\ \ \ \ $^3$ CIFAR Associate Fellow \\
}

\date{}

\newtheorem{theorem}{Theorem}[section]

\newtheorem{lemma}[theorem]{Lemma}
\newtheorem{remark}{Remark}[section]

\newtheorem{definition}{Definition}[section]

\newtheorem{varlemmainner}{Lemma}
\newenvironment{varlemma}[1]{%
  \renewcommand\thevarlemmainner{#1}%
  \varlemmainner
}{\endvarlemmainner}

\newtheorem{vartheoreminner}{Theorem}
\newenvironment{vartheorem}[1]{%
  \renewcommand\thevartheoreminner{#1}%
  \vartheoreminner
}{\endvartheoreminner}

\newtheoremstyle{subdefinition}
  {0pt}       
  {0pt}       
  {\upshape}  
  {0pt}       
  {\bfseries} 
  {.}         
  {5pt plus 1pt minus 1pt} 
  {}          
\theoremstyle{subdefinition}

\begin{document}
\maketitle
\begin{abstract}
Despite strong performance on a variety of tasks, neural sequence models trained with maximum likelihood have been shown to exhibit issues such as length bias and degenerate repetition. 
We study the related issue of receiving infinite-length sequences from a recurrent language model 
when using common decoding algorithms. 
To analyze this issue, we first define inconsistency of a decoding algorithm, meaning that the algorithm can yield an infinite-length sequence that has zero probability under the model.
We prove that commonly used incomplete decoding algorithms -- greedy search, beam search, top-$k$ sampling, and nucleus sampling -- are inconsistent, despite the fact that recurrent language models are trained to produce sequences of finite length.
Based on these insights, we propose two remedies which 
address inconsistency: consistent variants of top-$k$ and nucleus sampling, and a self-terminating recurrent language model.
Empirical results show that inconsistency occurs in practice, and that the proposed methods prevent inconsistency.
\end{abstract}

\section{Introduction}

Neural sequence models trained with maximum likelihood estimation (MLE) have become a standard approach to modeling sequences in a variety of natural language applications 
such as machine translation \cite{bahdanau2014neural}, dialogue modeling \cite{vinyals2015neural}, and language modeling \cite{radford2018gpt2}.
Despite this success, MLE-trained neural sequence models have been shown to exhibit issues such as length bias \cite{sountsov2016length,stahlberg2019nmt} and degenerate repetition \cite{holtzman2019curious}. 
These issues are suspected to be related to the maximum likelihood objective's local normalization, which results in a discrepancy between the learned model's distribution and the distribution induced by the decoding algorithm used to generate sequences \cite{lafferty2001conditional,andor2016globally}. 
This has prompted the development of alternative decoding methods \cite{wu2016google,holtzman2019curious} 
and training objectives \cite{murray2018correcting,welleck2019neural}. 
In this paper, we formalize and study this discrepancy between the model and the decoding algorithm.

We begin by formally defining \textit{recurrent neural language models}, a family that encompasses neural models used in practice, such as recurrent neural networks \cite{elman1990finding,cho2014properties,hochreiter1997lstm}, and transformers \cite{vaswani2017attention}.
Next, we formally define a decoding algorithm -- a function that induces a distribution over sequences given a recurrent language model and a context distribution -- which is used to obtain probable sequences from a model.
In this paper, we show that the distribution induced by a decoding algorithm can contradict this intended use; instead, the decoding algorithm may return improbable, infinite-length sequences.

Our main finding is that a sequence which receives zero probability under a recurrent language model's distribution can receive nonzero probability under the distribution induced by a decoding algorithm.
This occurs when the recurrent language model always ranks the sequence termination token outside of the set of tokens considered at each decoding step, yielding an infinite-length, zero probability sequence.
This holds whenever the decoding algorithm is \textit{incomplete}, in the sense that the algorithm excludes tokens from consideration at each step of decoding, which is the case for common methods such as greedy search, beam search, top-$k$ sampling \cite{fan2018hierarchical}, and nucleus sampling \cite{holtzman2019curious}. 
We formalize our main finding using the notion of \textit{consistency} \cite{chen2017recurrent} -- whether a distribution assigns probability mass only to finite sequences -- and prove that 
a consistent recurrent language model paired with an incomplete decoding algorithm can induce an inconsistent sequence distribution. 

Based on the insight that inconsistency occurs due to the behavior of the termination token under incomplete decoding, we develop two methods for addressing inconsistency.
First, we propose \textit{consistent sampling methods} which guarantee that the termination token is not excluded from selection during decoding. 
Second, we introduce a \textit{self-terminating recurrent language model} which ensures that the termination token is eventually ranked above all others, guaranteeing consistency under incomplete decoding.

To empirically measure inconsistency, we decode sequences from trained recurrent language models and measure the proportion of sequences with lengths far exceeding the maximum training sequence length.  
Our experiments on the Wikitext2 dataset \cite{merity2016pointer} suggest that inconsistency occurs in practice when using incomplete decoding methods, while the proposed consistent sampling methods and self-terminating model parameterization prevent inconsistency and maintain language modeling quality.

The theoretical analysis reveals defects of existing decoding algorithms, providing a way to develop 
future models, inference procedures, and learning algorithms.  
We present methods related to sampling and model parameterization, but there are more directions for future investigation; we close with directions related to sequence-level learning.

\section{Background}

We begin our discussion by establishing background definitions.
First, we define a sequence which is the main object of our investigation. 

\begin{definition}[Sequence]
A sequence $Y$ is an ordered collection of items from a predefined finite vocabulary $V$. A sequence of finite length always ends with a special token $\eos\in V$ that only appears at the end of a sequence.
\end{definition}

Each model we consider generates a sequence conditioned on context information, such as a prefix in sentence completion.
To consider this, we define a context distribution.

\begin{definition}[Context distribution]
A context distribution $p(C)$ is a probability distribution defined over a set $\mathcal{C}$. An element $C\in \mathcal{C}$ is called a context.
\end{definition}

\subsection{Recurrent Language Models}

A recurrent language model is an autoregressive model of a sequence distribution, where each conditional probability is parameterized with a neural network. Importantly, we assume that all tokens in a sequence are dependent on each other under a recurrent language model. This allows us to avoid cases in which the model degenerates to a Markovian language model, such as an $n$-gram model with a finite $n$.  

\begin{definition}[Recurrent language model]
\label{def:rlm}
A recurrent language model $p_\theta$ is a neural network that computes the following 
at each time step:
\begin{align*}
    p_{\theta}(y_t=v \,|\, y_{<t}, C) 
    = \frac{\exp(u_v^\top h_t + c_v)}
        {\sum_{v' \in V} \exp(u_{v'}^\top h_t + c_{v'})},
\end{align*}
where
$h_t = f_{\theta}(y_t, h_{t-1})$
and
$h_0 = g_{\theta}(C)$, and $u,c,\theta$ are parameters. 
A recurrent language model thereby computes the probability of a sequence $Y=(y_1, \ldots, y_T)$ by
\begin{align*}
    p_{\theta}(Y\,|\,C) = \prod_{t=1}^T p_{\theta}(y_t\,|\,y_{<t}, C),
\end{align*}
where $y_{<t}=(y_1,\ldots,y_{t-1})$.
This distribution satisfies
    $y_i \not\!\perp\!\!\!\perp y_{j}\,|\,C,\ \ \forall i < j$.
\end{definition}

Practical variants of the recurrent language model differ by the choice of transition function $f_{\theta}$ \cite{elman1990finding,hochreiter1997lstm,cho2014properties,vaswani2017attention}.
The use of softmax~\citep{bridle1990probabilistic} implies that every unique token in the vocabulary is considered at every location of a sequence. 
\begin{remark}
\label{remark:softmax}
Under the conditional distribution of a recurrent LM, every token $v\in V$ is assigned a positive probability, implying that $0 < p_\theta (v\,|\,y_{<t}, C) < 1.$
Any finite sequence is 
probable 
under a recurrent LM under any context, 
i.e., $p_{\theta}(Y\,|\,C) > 0$ for any sequence $Y$ of finite length.
\end{remark}

\subsection{Decoding Algorithms}

Because it is intractable to decode the most probable sequence, 
it is necessary in practice to use an approximate decoding algorithm. 

\begin{definition}[Decoding algorithm]
A decoding algorithm $\mathcal{F}(p_{\theta}, C)$ is a function
that generates a sequence $\tilde{Y}$ given a recurrent language model $p_{\theta}$ and context $C$.
Let $q_{\mathcal{F}}$ denote the distribution induced by the decoding algorithm $\mathcal{F}$.
\end{definition}

We consider two families of  decoding algorithms.
In our analysis we only consider algorithms that decode in a single pass, forward in time, without modifying previously selected tokens.

\paragraph{Stochastic decoding.} 
The first family consists of stochastic algorithms. Among them, ancestral sampling is asymptotically unbiased and can be used for 
finding the most probable sequence,
although with high variance.

\begin{definition}[Ancestral sampling]
Ancestral sampling $\mathcal{F}_{\text{anc}}$ generates a sequence from a recurrent language model $p_{\theta}$ given context $C$ by recursively sampling from $p_{\theta}(y_t\,|\,\tilde{y}_{<t}, C)$ until $\tilde{y}_t = \left<\text{eos}\right>$: $\tilde{y}_t \sim p_{\theta}(y_t\,|\,\tilde{y}_{<t}, C).$
\end{definition}

To avoid the high variance, two approximate stochastic decoding algorithms have recently been proposed and tested with recurrent language models. 
Top-$k$ sampling considers only a subset of the $k$ most probable tokens from the vocabulary at a time, while nucleus sampling considers only the minimal subset of most probable tokens whose total probability is higher than a predefined threshold. 

\begin{definition}[Top-$k$ sampling~\citep{fan2018hierarchical}]
Top-$k$ sampling $\mathcal{F}_{\text{top-k}}$ generates a sequence from a recurrent language model $p_{\theta}$ given context $C$ by recursively sampling from: 
\begin{align*}
    q(v) \propto 
    \begin{cases}
    p_{\theta}(v\,|\,y_{<t}, C), & \text{if } v \in V_k, \\
    0, & \text{otherwise.}
    \end{cases}
\end{align*}
where $V_k = \underset{v'}{\arg\text{top-k}}\ p_{\theta}(v'\,|\,y_{<t}, C)$.
\end{definition}

\begin{definition}[Nucleus sampling~\citep{holtzman2019curious}]
\label{def:nucleus}
Nucleus sampling $\mathcal{F}_{\text{nuc-}\mu}$ generates a sequence from a recurrent language model $p_{\theta}$ given context $C$ by recursively sampling from the following proposal distribution.
Let $v_1,\ldots,v_{|V|}$ denote tokens in $V$ such that $p_{\theta}(v_i\,|\,y_{<t},C) \geq p_{\theta}(v_j\,|\,y_{<t},C)$ for all $i < j$, and define
\begin{align*}
    q(v) \propto 
    \begin{cases}
      p_{\theta}(v\,|\,y_{<t}, C), & \text{if } v \in V_{\mu}, \\
      0, & \text{otherwise},
    \end{cases}
\end{align*}
where $V_{\mu} = \left\{ v_1, \cdots, v_{k_\mu} \right\}$ with
\begin{align*}
    k_{\mu} = \min 
        \left\{ k
            \,\left\vert\
            \sum_{i=1}^{k} p_{\theta}(v_i\,|\,y_{<t}, C) > \mu 
            \right.
        \right\}.
\end{align*}
\end{definition}

\paragraph{Deterministic decoding.}

The other family consists of deterministic decoding algorithms, where a token is selected deterministically according to a rule at each decoding step. The most naive algorithm, called greedy decoding, simply takes the most probable token at each step. 

\begin{definition}[Greedy decoding]
Greedy decoding $\mathcal{F}_{\text{greedy}}$ generates a sequence from a recurrent language model $p_{\theta}$ given context $C$ by recursively selecting the most likely token from $p_{\theta}(y_t | \tilde{y}_{<t}, C)$ until $\tilde{y}_t = \left<\text{eos}\right>$:
\begin{align*}
    \tilde{y}_t = \arg\max_{v\in V} \log p_{\theta}(y_t=v\,|\,\tilde{y}_{<t}, C).
\end{align*}
\end{definition}

In contrast to greedy decoding, \textit{beam search} with width $k$, $\mathcal{F}_{\text{beam-k}}$, operates on the level of partial sequences or prefixes.
Starting from a set of empty prefixes, at each iteration a new prefix set is formed by expanding each prefix with each possible token, then choosing the $k$ highest scoring expanded prefixes; refer to Appendix \ref{apx:sec-defs} for a formal definition.

\paragraph{Incompleteness.} 

Other than ancestral sampling, the decoding algorithms above are {\it incomplete} in that they only consider a strict subset of the full vocabulary $V$ at each time step, aside from the trivial case of $k=|V|$.\footnote{
Nucleus sampling is incomplete when for every context $C$ and prefix $y_{<t}$, $\min_{v \in V} p_{\theta}(v | y_{<t}, C) < 1-\mu.$ }

\begin{definition}[Incomplete Decoding]
\label{def:incomplete-decoding}
A decoding algorithm $\mathcal{F}$ is incomplete when for each context $C$ and prefix $y_{<t}$, there is a strict subset $V'_t\subsetneq V$ such that
\begin{align*}
    \sum_{v \in V'_t}
    q_{\mathcal{F}}(y_t=v\,|\,y_{<t},C)=1.
\end{align*}
\end{definition}

\section{Consistency of a Decoding Algorithm}
\paragraph{Definition of consistency.}

A recurrent language model $p_{\theta}$ may assign a positive probability to an infinitely long sequence, in which case we call the model inconsistent.
This notion of consistency was raised and analyzed earlier, for instance by \citet{booth1973applying} and \citet{chen2017recurrent}, in terms of whether the distribution induced by $p_{\theta}$ is concentrated on finite sequences. 
We extend their definition to account for the context $C$.

\begin{definition}[Consistency of a recurrent language model]
    A recurrent language model is consistent under a context distribution $p(C)$ if $p_{\theta}(|Y|=\infty) = 0$.
    Otherwise, the recurrent language model is said to be inconsistent.
\end{definition}

Any sequence decoded from a consistent model for a given context is guaranteed to terminate.

\begin{lemma} \label{lemma:finite_seq}
    If a recurrent LM $p_{\theta}$ is consistent,
    $p_{\theta}(|Y|=\infty\,|\,C)$ = 0 for any probable context $C$.\footnote{Proofs of Lemmas~\ref{lemma:finite_seq}-\ref{lemma:decoded_sequence_probable} are in Appendix~\ref{apx:sec-proofs-sec3}.}
\end{lemma}

Next, we establish a practical condition under which a recurrent language model is consistent. 

\begin{lemma}\label{lemma:bddrlm}
    A recurrent LM $p_{\theta}$ is consistent if $\|h_t\|_p$ is uniformly bounded for some $p\geq1$.
\end{lemma}
\begin{proof}[Proof sketch]
If $\|h_t\|_p$ is bounded, then each $u_v^\top h_t$ is bounded, hence $p_{\theta}(\eos | y_{<t}, C)>\xi>0$ for a constant $\xi$. Thus $p_{\theta}(|Y|=\infty) \leq \lim_{t\to\infty} (1 - \xi)^t = 0$, meaning that $p_{\theta}$ is consistent.
\end{proof}

Although this condition is practical because layer normalization or bounded activation functions \cite{elman1990finding,cho2014properties,vaswani2017attention} result in bounded $h_t$,
we show that even if a recurrent language model is consistent, a decoding algorithm may produce an infinite-length sequence.
We formalize this discrepancy using the consistency of a decoding algorithm. 

\begin{definition}[Consistency of a decoding algorithm]
\label{def:consistency}
A decoding algorithm $\mathcal{F}$ is consistent with respect to a consistent recurrent language model $p_{\theta}$ under a context distribution $p(C)$ if the decoding algorithm $\mathcal{F}$ preserves the consistency of the model $p_{\theta}$, that is,
$q_{\mathcal{F}}(|Y|=\infty)=0$.
\end{definition}

When a consistent recurrent language model $p_{\theta}$ and a decoding algorithm $\mathcal{F}$ induce a consistent distribution $q_{\mathcal{F}}$, 
we say that $p_{\theta}$ paired with $\mathcal{F}$ is consistent.
For instance, any consistent recurrent language model paired with ancestral sampling is consistent, 
because the induced distribution $q_{\mathcal{F}_{\text{anc}}}$ is the same as the distribution of the original model.
We also have an analogue of Lemma~\ref{lemma:finite_seq}.

\begin{lemma}
\label{lemma:decoded_sequence_probable}
A consistent decoding algorithm with respect to a consistent recurrent LM decodes only probable sequences. That is, if $q_{\mathcal{F}}(Y\,|\,C)>0$, then $p_{\theta}(Y\,|\,C)>0$ for any probable context $C$.
\end{lemma}

\paragraph{Inconsistency of incomplete decoding.}
Any incomplete decoding algorithm (Definition \ref{def:incomplete-decoding}) can be inconsistent regardless of the context distribution, because there is a recurrent LM that places \eos\ outside of $V'_t$ at every step of decoding.
To show this, we construct a consistent recurrent language model whose distribution induced by an incomplete decoding algorithm is inconsistent. 

\begin{figure}
    \centering
    \includegraphics[width=0.9\columnwidth]{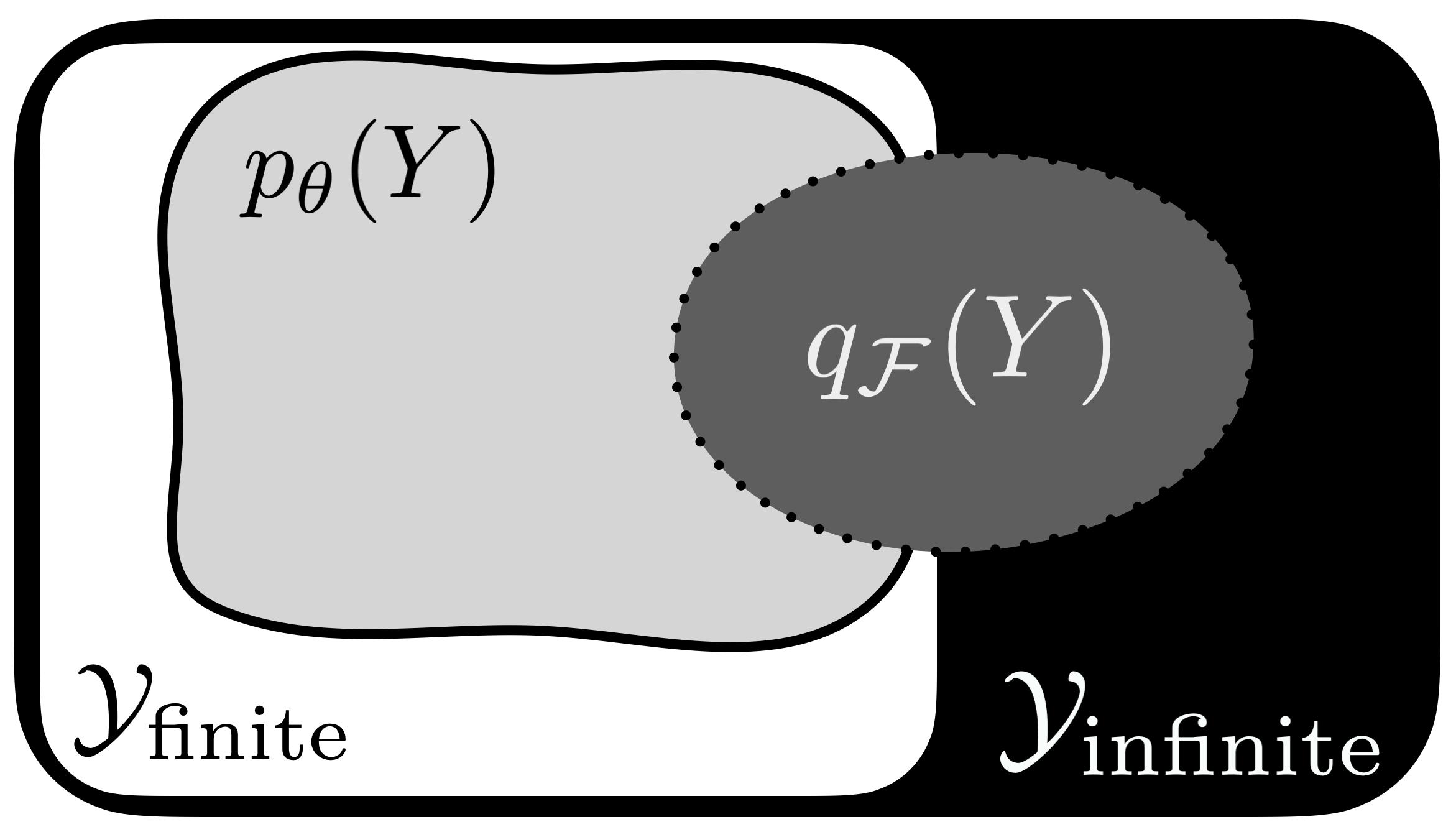}
    \caption{A depiction of the model's sequence distribution (light grey, solid border) and the decoder's induced sequence distribution (dark grey, dotted border).
    The white and black rectangles depict the set of all finite and infinite sequences, respectively.
    We prove that under practical conditions, any incomplete decoding algorithm may be inconsistent with respect to a consistent model, as depicted.
    }
    \label{fig:consistency}
\end{figure}

\begin{theorem}[Inconsistency of an incomplete decoding algorithm]
\label{theorem:inconsistent-incomplete}
There exists a consistent recurrent LM $p_{\theta}$ from which an incomplete decoding algorithm $\mathcal{F}$, that considers only up to $(|V|-1)$-most likely tokens according to $p_{\theta}(y_t\,|\,y_{<t},C)$ at each step $t$, finds an infinite-length sequence $\tilde{Y}$ with probability 1, i.e., $q_{\mathcal{F}}(|Y|=\infty)=1$.
\end{theorem}
\begin{proof}
We prove this theorem by constructing a $\tanh$ recurrent network. We define the recurrent function $f_{\theta}$ as
\begin{align*}
    h_t & = f_{\theta}(y_t, h_{t-1}) \\
    & =
    \tanh\left(\left[
    \begin{array}{c | c}
    W_h  & \mathbf{0}  \\
    \hline
    \mathbf{0} & 
    I
    \end{array}
    \right]
    h_{t-1}
    +
    \left[
    \begin{array}{c}
    \mathbf{0} \\
    \hline
    e(y_{t})
    \end{array}
    \right]
    \right), 
\end{align*}
where $e(y_{t}) \in \mathbb{R}^{|V|}$ is a one-hot representation of $y_t$, $W_h \in \mathbb{R}^{d \times d}$ where every entry is positive, and $I$ is an identity matrix of size $|V| \times |V|$.
$h_0 = g_{\theta}(C)$ is constructed to consist of positive values only.
Because each element of $|h_t|$ is bounded by 1, the constructed recurrent language model $p_{\theta}$ is consistent by Lemma~\ref{lemma:bddrlm}.

We set $u_v$ (see Definition~\ref{def:rlm}) to
\begin{align*}
    u_v = \left[ 
    \begin{array}{c}
         \bar{u}_v  \\
         \hline
          e(v)
    \end{array}
    \right], & \quad u_{\eos}
    = \left[
    \begin{array}{c}
         \bar{u}_{\eos}  \\
         \hline
         e(\eos)
    \end{array}
    \right],
\end{align*}
where $v\neq\eos$, all elements of $\bar{u}_v$ are positive,  all elements of $\bar{u}_{\eos}$ are negative, and $e(v)$ is a one-hot representation of $v$. $c_v$ is set to zero.

This defines a valid recurrent language model (Definition~\ref{def:rlm}), since the conditional distribution at each time $t$ is influenced by all the previous tokens. 
More specifically, the logit of a token $v$ depends on $\sum_{t'=1}^t \mathds{1}(y_{t'} = v)$, where $\mathds{1}$ is an indicator function.

This recurrent language model always outputs positive logits for non-\eos\ tokens, and outputs negative logits for the \eos\ token. This implies $p(\eos|\,y_{<t}, C) < p(v\,|\,y_{<t}, C)$ for all $v \in V \backslash \left\{\eos\right\}$. 
This means that \eos\ is always ranked last at each time step, so an incomplete decoding algorithm that considers at most $(|V|-1)$ most probable tokens at each time step from $p_{\theta}(y_t\,|\,y_{<t}, C)$ cannot decode \eos\ and thus always decodes an infinitely long sequence $\hat{Y}$,
i.e., $q_{\mathcal{F}}(|Y|=\infty\,|\,C)=1$ for any context $C$.
It yields $q_{\mathcal{F}}(|Y|=\infty)=1$, while $p_{\theta}(|Y|=\infty) = 0$ due to consistency of the model $p_{\theta}$. 
\end{proof}
Greedy decoding, beam search, top-$k$ sampling, and nucleus sampling are all inconsistent according to this theorem.

\section{Fixing the inconsistency}
In this section, we consider two ways to prevent inconsistency arising from incomplete decoding algorithms. 
First, 
we introduce consistent versions of top-$k$ and nucleus sampling. 
Second, we introduce the \textit{self-terminating recurrent language model}, which is consistent when  paired with any of the  decoding algorithms  considered in this paper.

\subsection{Consistent Sampling Algorithms}
\label{ssec:consistent-sampling}
The proof of Theorem~\ref{theorem:inconsistent-incomplete} suggests that the inconsistency of incomplete decoding algorithms arises from the fact that \eos\ may be excluded indefinitely from the set of top-ranked tokens. We propose a simple modification to top-$k$ and nucleus sampling that forces \eos\ to be included at each step of decoding. 
First, we give a condition for when a particular model $p_{\theta}$ paired with a decoding algorithm $\mathcal{F}$ is consistent.

\begin{theorem}\label{theorem:suff_cond_of_consistency}
Suppose a recurrent LM $p_{\theta}$ has uniformly bounded $\|h_t\|_p$ for some $p\geq 1$.
If a decoding algorithm $\mathcal{F}$ satisfies $q_{\mathcal{F}}(\eos|\,y_{<t}, C) \geq p_{\theta}(\eos|\,y_{<t}, C)$ for every prefix $y_{<t}$ and context $C$,
then the decoding algorithm $\mathcal{F}$ is consistent with respect to the model $p_{\theta}$.\footnote{See Appendix \ref{apx:sec-proofs-sec4} for the proof.}
\end{theorem}

We define consistent variants of top-$k$ and nucleus sampling which satisfy this condition.

\begin{definition}[Consistent top-$k$ sampling]
Consistent top-$k$ sampling is top-$k$ sampling with the following modified proposal distribution:
\begin{align*}
    q(v) \propto
    \begin{cases}
      p_{\theta}(v|y_{<t}, C),&\text{if } v \in V', \\
      0, & \text{otherwise},
    \end{cases}
\end{align*}
where
$V' = \left\{ \eos \right\} \cup \underset{v'}{\arg\text{top-k}}\ p_{\theta}(v'\,|\,y_{<t}, C)$.
\end{definition}

\begin{definition}[Consistent nucleus sampling]
Consistent nucleus sampling is nucleus sampling with the following modified proposal distribution:
\begin{align*}
    q(v) \propto 
    \begin{cases}
      p_{\theta}(v\,|\,y_{<t}, C), & \text{if } v \in V_{\mu} \cup \{\eos\}, \\
      0, & \text{otherwise}.
    \end{cases}
\end{align*}
\end{definition}

The induced probability of \eos\ under these two algorithms is always equal to or larger than the model's probability. 
By Theorem~\ref{theorem:suff_cond_of_consistency}, these algorithms are consistent with respect to any consistent recurrent language model.

\begin{figure}
    \centering
    \includegraphics[width=.98\columnwidth]{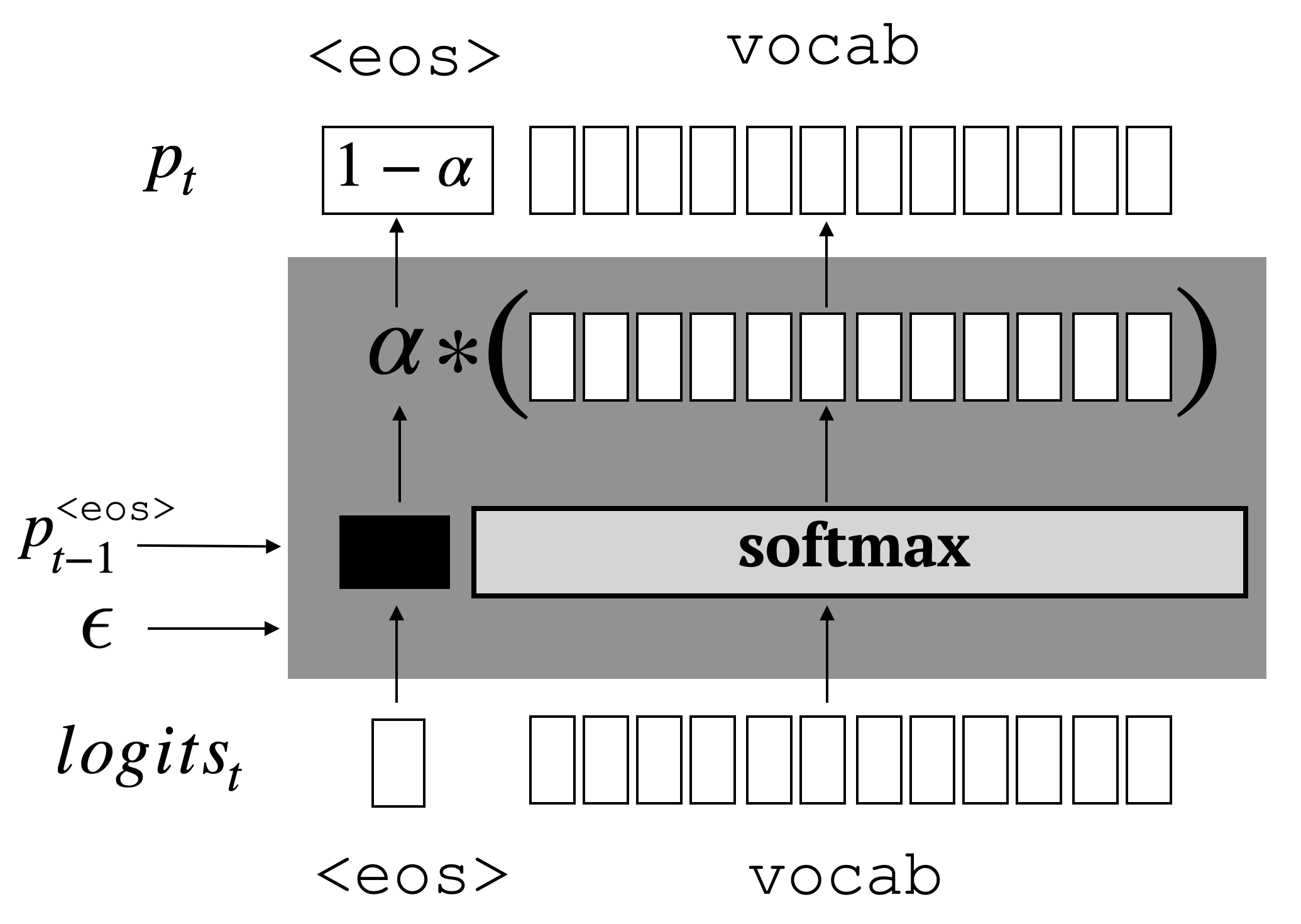}
    \caption{The self-terminating recurrent LM uses the layer shown in grey instead of the standard softmax layer.
    The layer takes logits ($u^\top_\cdot h_t$), the previous step's $\eos$ probability ($p_{t-1}^{\eos}$), and a hyper-parameter $\epsilon\in (0,1)$.
    The layer computes $\alpha$ using Definition \autoref{def:strlm}, which determines the $\eos$ probability ($p_t^{\eos} \in (\epsilon,1)$), and guarantees that $p_t^{\eos}>p_{t-1}^{\eos}$.
    The remaining probability mass is allocated to the non-$\eos$ tokens.
    }
    \label{fig:st-softmax}
\end{figure}

\subsection{Self-Terminating Recurrent LM}
\label{ssec:strlm}
Although these consistent sampling algorithms can be used with any recurrent language model, their stochastic nature may not be suitable for finding a single, highly probable sequence. 
To avoid this limitation, we propose the \textit{self-terminating recurrent language model} (STRLM).

\begin{definition}[Self-terminating recurrent language model]
\label{def:strlm}
A self-terminating recurrent language model computes the following conditional probability at each time step:
\begin{align*}
    p_{\theta}(v\,|\,y_{<t}, C) =&
    \begin{cases}
        1-\alpha(h_t), \quad\quad\quad v=\eos, \\
        \frac{\alpha(h_t) \exp(u_v^\top h_t + c_v)}
          {\sum_{v' \in V'} \exp(u_{v'}^\top h_t + c_{v'})}, 
    \end{cases}\\
    \alpha(h_0)=\sigma(u_{\left<\text{eos}\right>}^\top& h_0),\\
    \alpha(h_t)=\sigma(u_{\left<\text{eos}\right>}^\top& h_t) \left[1- p_{\theta}(\left<\text{eos}\right>|y_{<t-1}, C)\right],
\end{align*}
with $\sigma: \mathbb{R} \to [0,1-\varepsilon]$ and $\varepsilon \in (0,1)$. 
$h_t$ is computed as in the original recurrent LM.
\end{definition}
The underlying idea is that the probability of \eos\ increases monotonically, since
\begin{align*}
    p_{t}^{\eos} & = 1 - \prod_{t'=0}^t \sigma(u_{\left<\text{eos}\right>}^\top h_{t'}).
\end{align*}
Consequently, the STRLM is consistent when paired with greedy decoding or beam search; see Appendix \ref{apx:sec-proofs-sec4} for formal statements and proofs.

\section{Empirical Validation}
The theoretical results rely on the existence of a model that results in inconsistency; it remains to be shown that inconsistency with respect to incomplete decoding occurs with recurrent language models encountered in practice. 
Moreover, while the proposed methods
carry theoretical guarantees in terms of consistency, we must check whether they retain language modeling quality. 
To do so, we perform experiments using a sequence completion task.
In each experiment, we use the beginning of a sequence as context, then decode continuations from a trained recurrent LM and measure the proportion of non-terminated sequences in order to approximately measure inconsistency.
The first experiment (\S\ref{ssec:expr-inconsistency}) shows that inconsistency occurs in practice, and the second experiment (\S\ref{ssec:expr-strlm}) shows the effectiveness of the proposed approaches.
Our third experiment (\S\ref{ssec:gpt-2}) shows that inconsistency also occurs frequently in GPT-2, a large-scale transformer language model.\footnote{Code available at \url{https://github.com/uralik/consistency-lm}.}

\paragraph{Sequence completion.} 
We evaluate recurrent language models on a sequence completion task, which has previously been used to evaluate the effectiveness of sequence models, e.g., \citet{sutskever2011generating,graves2013generating,radford2018gpt2,holtzman2019curious,welleck2019neural}. 
Sequence completion is a general setting for studying the behavior of language models, encompassing machine translation \cite{bahdanau2014neural}, story generation \cite{fan2018hierarchical}, and dialogue modeling \cite{vinyals2015neural}.
The task consists of decoding a continuation $\hat{Y}\sim\mathcal{F}(p_{\theta}, C)$ given a length-$k$ prefix  $C=(c_1,\ldots,c_k)$,
resulting in a completion $(c_1,\ldots,c_k,\hat{y}_1\ldots,\hat{y}_T)$.

\paragraph{Dataset.} 
Our first two experiments use Wikitext2 \cite{merity2016pointer}, which consists of paragraphs from English Wikipedia, since it has frequently been used to evaluate language models \cite{grave2017improving,melis2018on,merity2018regularizing}. 
We consider both word and BPE\footnote{\url{github.com/huggingface/tokenizers}} tokenization.
We split each paragraph into sentences using Spacy\footnote{\url{https://spacy.io/}}.
We split each sequence, using the first $k$ tokens as a context and the remaining tokens as a continuation.
To ensure that each sequence contains a prefix, we prepend padding tokens to make it length $k$.
Special \bos\ and \eos\ tokens are inserted at the beginning and end of each sequence. 
We use $k=10$. 
Table~\ref{tbl:dataset-stat} contains dataset statistics.

\paragraph{Context distribution.} We define empirical context distributions with prefixes from the train, valid, and test sets: 
    $p(C;\mathcal{D}) = \frac{1}{|\mathcal{D}|} \sum_{n=1}^{|\mathcal{D}|} \mathds{1}(C = C^{(n)})$,
where $\mathcal{D}=\{(C^{(n)},Y^{(n)})\}_{n=1}^{N}$ is a dataset split.
 
\paragraph{Evaluation metrics.} 
We use finite sequences to approximately measure the consistency of a model paired with a decoding algorithm, since decoding an infinite-length sequence is impossible.
We use the proportion of decoded continuations that are longer than a predefined limit, 
\begin{align*}
    r_L = \frac{1}{|\mathcal{D}|}\sum_{n=1}^{|\mathcal{D}|} \mathds{1}(|\hat{Y}^{(n)}| \geq L),
\end{align*}
where $\hat{Y}^{(n)}\sim \mathcal{F}(p_{\theta}, C^{(n)})$ for each context $C^{(n)}$ in $\mathcal{D}$.
We call $r_L$ the \textit{non-termination ratio} of the decoding algorithm $\mathcal{F}$ for an underlying model and context distribution.
A value of $r_L$ greater than zero means that some sequences did not terminate within $L$ steps. 
When $L$ is infinity, this implies that the model paired with the decoding algorithm is inconsistent. 
In practice, we use a finite $L$ that is substantially larger than the maximum training sequence length, and we interpret a non-zero $r_L$ as evidence that the model paired with the decoding algorithm is inconsistent. 
We use $L=1500$, more than 10 times the max training sequence length.

In each experiment, we report the mean and standard deviation of metrics across 10 independent initializations. Unless specified otherwise, we report metrics using the test context distribution, since the train, valid, and randomly generated context distributions had similar results.

\paragraph{Training.} We train recurrent language models for sequence completion with maximum likelihood, using the loss 
    $\mathcal{L}(p_{\theta}, Y)
    =-\sum_{t=1}^T\log p_{\theta}(y_t\,|\,y_{<t}, c_1,\ldots,c_k)$,
where $Y=(c_1,\ldots,c_k,y_1,\ldots,y_T)$.
This amounts to running the full training sequence through a recurrent model and zeroing the loss for the first $k$ tokens, so that the first $k$ steps correspond to learning a $g_{\theta}$ that encodes the context. 

\paragraph{Models.} We consider recurrent neural networks with hyperbolic tangent activations \citep[$\tanh$-RNN;][]{elman1990finding} and LSTM units \citep[LSTM-RNN;][]{hochreiter1997lstm}. 
We perform an initial hyper-parameter sweep and select the best set of hyper-parameters for each of $\tanh$-RNN and LSTM-RNN based on the validation perplexities.\footnote{Refer to Appendix \ref{apx:sec-additional} for the hyper-parameter ranges.} 
With this best set of hyperparameters, we train each of these models with 10 different initializations.
The choice of $\tanh$ and LSTM RNNs implies that all of the recurrent language models that we train are consistent according to Lemma \ref{lemma:bddrlm}. 
Our LSTM models achieve similar test perplexity ($91.86 \pm 0.4$, word tokenization) to those reported in previous work \cite{merity2018regularizing}; see Appendix \ref{apx:sec-additional}.

Additionally, we train self-terminating $\tanh$-RNN and LSTM-RNN variants (Definition \ref{def:strlm}) at various values of $\varepsilon$, which controls a lower bound on the termination probability at each step. We use $\sigma(x)=(1-\varepsilon)\cdot \text{sigmoid}(x)$.
We use the hyper-parameters selected in the preceding grid search. Below, we consider BPE tokenization; similar conclusions held for word tokenization.\footnote{Refer to Appendix for results with word tokenization.}
\begin{table}[t]
\centering
\small 

\begin{tabular}{lll}
\toprule
&\multicolumn{1}{l}{$\mathbf{\tanh}$\textbf{-RNN}} & \multicolumn{1}{l}{\textbf{LSTM-RNN}} \\ 
\midrule
\textbf{ancestral} & 0.00 $\pm$ 0.0 & 0.00 $\pm$ 0.0  \\
\midrule
\textbf{greedy} & 12.35 $\pm$ 5.18 & 1.53 $\pm$ 1.41\\
\textbf{beam-2} & 1.38 $\pm$ 0.95 & 0.07 $\pm$ 0.06  \\ 
\textbf{beam-4} & 0.25 $\pm$ 0.19 & 0.00 $\pm$ 0.01  \\
\midrule
\textbf{topk-2} & 0.01 $\pm$ 0.01 & 0.01 $\pm$ 0.01 \\
\textbf{topk-4} & 0.00 $\pm$ 0.0 & 0.00 $\pm$ 0.01 \\  
\textbf{nucleus-0.2} & 0.06 $\pm$ 0.02 & 0.13 $\pm$ 0.15 \\
\textbf{nucleus-0.4} & 0.04 $\pm$ 0.02 & 0.02 $\pm$ 0.01 \\
\midrule
\textbf{consistent topk-2} & 0.00 $\pm$ 0.0 & 0.00 $\pm$ 0.01 \\
\textbf{consistent topk-4} & 0.00 $\pm$ 0.0 & 0.00 $\pm$ 0.0 \\
\textbf{consistent nucleus-0.2} & 0.04 $\pm$ 0.02 & 0.01 $\pm$ 0.01 \\
\textbf{consistent nucleus-0.4} & 0.02 $\pm$ 0.02 & 0.01 $\pm$ 0.01\\
\bottomrule
\end{tabular}
\caption{Non-termination ratio ($r_L$ (\%)) of decoded sequences using ancestral sampling, incomplete, and consistent decoding methods.}
\label{tbl:rnn-results-decoding-bpe}
\end{table}

\subsection{Inconsistency of Recurrent LMs}
\label{ssec:expr-inconsistency}

In this experiment, we demonstrate evidence of inconsistency with incomplete decoding methods. 
Table~\ref{tbl:rnn-results-decoding-bpe} shows non-termination ratios for the recurrent language models using the decoding algorithms considered in this work.
Decoding with ancestral sampling always resulted in sequences that terminated within $L$ steps, since the induced distribution is the same as that of the consistent model.
On the other hand, the non-zero non-termination ratios for the incomplete decoding algorithms suggest inconsistency with respect to each algorithm, providing evidence for Theorem \ref{theorem:inconsistent-incomplete}. 

Using greedy decoding, roughly 12\% of all contexts resulted in a non-terminating continuation with the $\tanh$-RNN, and roughly 1\% with the LSTM-RNN. 
Nucleus sampling also produced non-terminating sequences with the $\tanh$-RNN (0.06\%, nuc-0.2) and LSTM-RNN (0.13\%, nuc-0.2).
Top-$k$ sampling yielded a small number of non-terminating samples. 
In general, non-termination approaches zero as $k$ and $\mu$ increase, since $\eos$ has a lower chance of being excluded.

Beam search produced non-terminating sequences with both the $\tanh$-RNN and LSTM-RNN models. 
This means that $\eos$ was outside of the top tokens (determined by the beam width) considered at each step, since in our experiments we terminated the beam search when a single beam prefix contained $\eos$. 
Larger beam widths reduce non-termination, similar to increasing $k$ or $\mu$.

\begin{table*}[h!]
\centering
\resizebox{\linewidth}{!}{
\begin{tabular}{lp{1.05\linewidth}}
\toprule
\textbf{Prefix} & \textit{One Direction delivered a performance of `` Kiss You}\\
\textbf{nucleus} & '' , and the album 's second album , `` The X @-@ Files '' , `` The A. '' , `` The Preder '' , `` We 've Have You '' , `` I 've You Wanna Stay '' , `` The Dream '' , `` The Bide '' , `` My Achievement '', `` The B. B. '' , `` A Life ''  $\ldots$ \\
\textbf{c-nucleus} & '' , and `` My Boo '' was released on September 29 , 2010 . \eos \\
\midrule
\textbf{Prefix} & \textit{Boulter starred in two films in 2008 ,}\\
\textbf{nucleus} & and the band 's music , and `` The Rise of Monkey '' , `` The One With the Way '' , `` The `` Always '' , '' `` Always Your '' , `` The Wift '' , `` The Baste '' , `` The Special With '' , `` The Way '' , `` The Special With You ''   $\ldots$ \\
\textbf{c-nucleus} & and the latter was released in the United States . \eos \\
\bottomrule
\end{tabular}
}
\end{table*}
\begin{table*}[h!]
\vspace{-4mm}
\centering
\resizebox{\linewidth}{!}{
\begin{tabular}{lp{1.05\linewidth}}
\textbf{Prefix} & \textit{This period of unhappiness was the making of}\\
\textbf{Baseline} &  the `` most important " of the `` mad " , and the `` `` most important " of the '' `` '' , `` the most important " , and the `` devil " , `` The " , `` The One " , `` The One " , `` The One " , `` The One " , `` The One " , `` The One " , `` The One " , `` The One " , `` The One " , `` The One " , `` The One " , `` The One " , `` The One " , `` The One " $\ldots$  \\
\textbf{STRLM} & the first commandment of the poem . \eos\\
\midrule
\textbf{Prefix} & \textit{Du Fu 's mother died shortly after he was}\\
\textbf{Baseline} &  a member of the Order of the Order of the Order of the Order of the Order of the Order of the Order of the Order of the Order of the Republic of the Republic of the Republic of the Republic of the Republic of $\ldots$  \\
\textbf{STRLM} & a member of the Order of the British Empire . \eos\ \\
\bottomrule
\end{tabular}
}
\caption{Continuations with consistent nucleus sampling ($\mu = 0.2$) and self-terminating LSTM ($\epsilon=10^{-3}$).}
\label{tbl:st-continuations}
\end{table*}

\subsection{Consistency of the Proposed Methods}
\label{ssec:expr-strlm}

\paragraph{Consistent sampling.} Table~\ref{tbl:rnn-results-decoding-bpe} shows that consistent nucleus and top-$k$ sampling (\S\ref{ssec:consistent-sampling}) resulted in only terminating sequences, except for a few cases that we attribute to the finite limit $L$ used to measure the non-termination ratio.
Consistent nucleus paired with $\tanh$-RNN did not reduce $r_L$ as much as when it was paired with LSTM-RNN.
Example continuations are shown in Table~\ref{tbl:st-continuations}.
On prefixes that led to non-termination with the baseline method, the quality tends to improve with the consistent variant since the continuation now terminates.
Note that since the model's non-\eos\ token probabilities at each step are only modified by a multiplicative constant, the sampling process can still enter a repetitive cycle (e.g., when the constant is close to 1), though it is guaranteed to terminate.

\paragraph{Self-terminating RLM.} As seen in Table \ref{tbl:st-results-bpe}, the self-terminating recurrent language models are consistent with respect to greedy decoding, at the expense of perplexity compared to the vanilla model.
The value of $\varepsilon$ from Definition \ref{def:strlm}, which controls a lower-bound on termination probability at each step, influences both $r_L$ and perplexity. 
When $\varepsilon$ is too large ($\varepsilon=10^{-2}$), perplexity degrades.
When $\varepsilon$ is too small ($\varepsilon=10^{-4}$), the lower-bound grows slowly, so \eos\ is not guaranteed to be top-ranked within $L$ steps, resulting in a positive $r_L$.
An $\varepsilon$ of $10^{-3}$ balanced  consistency and language modeling quality, with a zero non-termination ratio and perplexity within 8 points of the baseline.

As shown in Figure~\ref{fig:st_lstm_length}, the self-terminating model matches the data length distribution better than the baseline.
Example decoded sequences are shown in Table~\ref{tbl:st-continuations}. 
For prefixes that led to non-termination with the baseline, the self-terminating models yields finite sequences with reasonable quality.
The examples suggest that some cases of degenerate repetition \cite{holtzman2019curious,welleck2019neural} are attributed to inconsistency. 

\begin{table}[t]
\centering
\small 

\begin{tabular}{p{0.1cm}p{0.1cm}llllll}
\toprule
&ST&$\epsilon$&$r_L$ (\%) & \textbf{perplexity} \\
\midrule
\multirow{4}{*}{\rotatebox{90}{{$\tanh$-RNN}}}&\Checkmark&  $10^{-2}$ & 00.00 $\pm$ 0.0 & 229.09 $\pm$ 9.2 \\
&\Checkmark& $10^{-3}$ & 00.00 $\pm$ 0.0 & 191.63 $\pm$ 1.4 \\
&\Checkmark& $10^{-4}$ & 00.02 $\pm$ 0.02 & 188.36 $\pm$ 2.2 \\
& \text{\ding{55}} & -- & 12.35 $\pm$ 5.2 & 186.44 $\pm$ 1.4 \\
\midrule
\multirow{4}{*}{\rotatebox{90}{{LSTM}}}&\Checkmark& $10^{-2}$ & 0.00 $\pm$ 0.0 & 219.71 $\pm$ 9.2 \\
&\Checkmark& $10^{-3}$ & 0.00 $\pm$ 0.0 & 186.04 $\pm$ 1.6  \\
&\Checkmark& $10^{-4}$ & 0.18 $\pm$ 0.35 & 183.57 $\pm$ 2.3 \\
&\text{\ding{55}}& -- & 1.48 $\pm$ 1.43 & 178.19 $\pm$ 1.3 \\ 
\bottomrule
\end{tabular}
\caption{\label{tbl:st-results-bpe} Non-termination ratio $(r_L$ (\%)) of greedy-decoded sequences and test perplexity for STRLMs.}
\end{table}
\begin{figure}[ht!]
\includegraphics[width=0.45\textwidth]{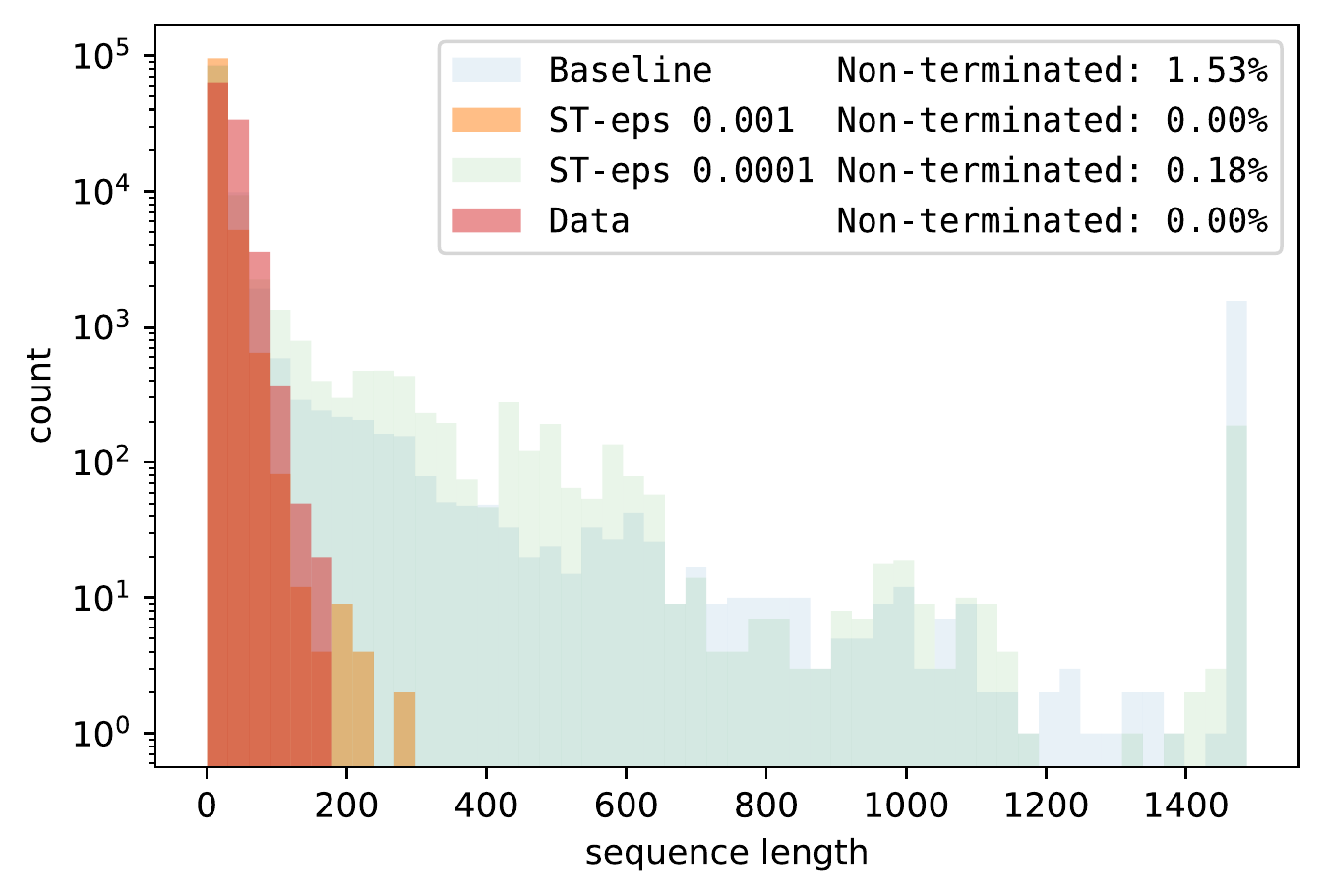}

\caption{Lengths of generated sequences using greedy decoding from vanilla and self-terminating LSTMs.}
\label{fig:st_lstm_length}
\end{figure}

\subsection{Inconsistency of GPT-2}
\label{ssec:gpt-2}
We perform a final experiment with GPT-2 117M, a transformer language model pre-trained with maximum likelihood on WebText, a collection of scraped web pages (see \citet{radford2018gpt2}).
GPT-2 has been observed to produce repetitive text with greedy and beam search \citep{holtzman2019curious}.

\paragraph{Experimental setup.}
We use the Wikitext-103 dataset \citep{merity2016pointer}, a large-scale collection of Wikipedia articles with over 100 million words and 260 thousand unique tokens.
We split the dataset into sequences according to the dataset's newline boundaries, then split each sequence into a context $C$ and continuation $Y$, resulting in a dataset of $(C, Y)$ pairs. 
Each continuation ends in a special $\eos$ token.
We use a context size of $k=10$ tokens, and discard sequences that are length $k$ or shorter.
The resulting dataset contains 874,556 training, 1,896 validation, and 2,162 test pairs.

We fine-tune the pre-trained GPT-2 model using maximum likelihood for 400k steps, and select the model state with the lowest validation perplexity (evaluated every 5k steps). 
Each training batch contains a maximum of 1024 total tokens.
We use the implementation and default hyper-parameters from the \texttt{transformers} library \citep{wolf2019huggingface}.
We fine-tune the self-terminating GPT-2 models in a similar manner, starting from the pre-trained GPT-2 model and using the same hyper-parameters.

Each model is evaluated using greedy decoding with a maximum sequence length of 500, which was selected so that each decoded validation batch could fit in GPU memory.
We define the non-termination ratio $(r_L)$ using $L=500$; this limit is more strict than the limit used in the preceding experiments (1500), yet still allows us to see large differences in generation behavior between the model and the ground truth (e.g. see Figure \ref{fig:gpt2_lentgth}).

\paragraph{Results.} Table \ref{tbl:gpt2-st-results} shows the non-termination ratio and perplexity of the baseline and self-terminating GPT-2 models.
The self-terminating variant prevents non-termination, at the cost of perplexity.
The model here uses $\epsilon=2.5\times 10^{-3}$, which we selected after observing that at higher values of $\epsilon$, e.g. $1.0\times 10^{-3}$, the self-terminating model generated sequences longer than the limit used to determine termination (500).
Figure \ref{fig:gpt2_lentgth} shows the length distributions of the baseline GPT-2 continuations and those of the self-terminating GPT-2.
The GPT-2 117M model generates many sequences at or near the maximum sequence length (500), unlike the ground-truth data.
Introducing self-termination shifts the mass towards shorter sequences, whose lengths are also present in the ground-truth data.

\begin{table}[t]
\centering
\small 

\begin{tabular}{lll}
\toprule
&$r_L$ (\%) & \textbf{perplexity} \\
\midrule
GPT2-117M    & 37.91 & 20.92\\
GPT2-117M ST & 00.00 & 27.25\\
\bottomrule
\end{tabular}
\caption{\label{tbl:gpt2-st-results} Non-termination ratio $(r_L$ (\%)) of greedy-decoded sequences and perplexity for GPT2-117M and the self-terminating variant (ST) on Wikitext-103.}
\end{table}

\section{Future Directions}
\label{sec:future}
The methods we proposed in this paper resolve inconsistency by changing the decoding algorithm or model parameterization.
Another approach is to address inconsistency in the \textit{learning} phase.
One interesting direction is to investigate whether the lack of decoding in maximum likelihood learning is a cause of inconsistency.
Maximum likelihood learning fits the model $p_{\theta}$ using the data distribution, whereas a decoded sequence from the trained model follows the distribution $q_{\mathcal{F}}$ induced by a decoding algorithm.
\textit{Sequence-level} learning, however, uses a decoding algorithm during training (e.g., \citet{ranzato2016seqeunce}),
which we hypothesize can result in a good sequence generator that is consistent with respect to incomplete decoding.

\section{Conclusion}
We extended the notion of consistency of a recurrent language model put forward by \citet{chen2017recurrent} to incorporate a decoding algorithm, and used it to analyze the discrepancy between a model and the distribution induced by a decoding algorithm.
We proved that incomplete decoding is inconsistent, and proposed two methods to prevent this: consistent decoding and the self-terminating recurrent language model. 
Using a sequence completion task, we confirmed that empirical inconsistency occurs in practice, and that each method prevents inconsistency
while maintaining the quality of generated sequences. 
We suspect the absence of decoding in maximum likelihood estimation  as a cause behind this inconsistency, and suggest investigating sequence-level learning as an alternative. 

\section*{Acknowledgements}

We thank Chris Dyer, Noah Smith and Kevin Knight for valuable discussions. 
This work was supported by NSF Award 1922658 NRT-HDR: FUTURE Foundations, Translation, and Responsibility for Data Science; Samsung Advanced Institute of Technology (Next Generation Deep Learning: from pattern recognition to AI); and Samsung Research (Improving Deep Learning using Latent Structure). 
KC thanks eBay and NVIDIA for their support.

\bibliographystyle{acl_natbib}
\bibliography{anthology,emnlp2020}

\clearpage

\appendix

\section{Additional Definitions}
\label{apx:sec-defs}
In contrast to greedy decoding, beam search with width $k$, $\mathcal{F}_{\text{beam-k}}$, operates on the level of partial sequences or prefixes.

\begin{definition}[Prefix]
\label{def-beam-hyp}
A prefix $\rho_t$ is an ordered collection of items from $V$. 
The score of a prefix is
\begin{align*}
s(\rho_t) = \sum_{\tau=1}^{t} \log p_\theta(y_{\tau} = \rho_t[\tau]\,|\,\rho_t[<\tau],C),
\end{align*}
where $\rho_t[\tau]$ is a token at time $\tau$ from $\rho_t$.
\end{definition}

Starting from a set of empty prefixes, at each iteration a new prefix set is formed by expanding each prefix, then choosing the $k$ highest scoring expanded prefixes.

\begin{definition}[Beam search]
\label{def:beam}
Beam search with width $k$, $\mathcal{F}_{\text{beam}-k}$, generates a sequence from a recurrent language model $p_{\theta}$ by maintaining a size-$k$ prefix set $\mathrm{P}_t^{\text{top}}$. 
Starting with $P_0^{top}=\varnothing$, at each iteration $t\in \{1,2,\ldots\}$ beam search forms a new prefix set $\mathrm{P}_t^{\text{top}}$ by expanding the current set, $\mathrm{P}_t = \bigcup_{\rho \in \mathrm{P}_{t-1}^{\text{top}}} \{\rho \circ v\, |\, v\in V\}$ (where $\rho \circ v$ is concatenation),
then choosing the $k$ highest scoring elements: 
$\mathrm{P}_t^{\text{top}} = \underset{\rho \in \mathrm{P}_t}{\arg\text{top-k}}\ s(\rho).$ 
Any $\rho \in \mathrm{P}_t^{\text{top}}$ ending with $\left<\text{eos}\right>$ is restricted from being expanded further, and is added to a set $S$. Beam search ends when $S$ contains $k$ sequences, and returns the highest scoring sequence in $S$. 
\end{definition}

\section{Proof of Lemmas in Section 3}
\label{apx:sec-proofs-sec3}
\begin{varlemma}{3.1}
    If a recurrent language model $p_{\theta}$ is consistent,
    $p_{\theta}(|Y|=\infty\,|\,C)=0$ for any probable context $C$.
\end{varlemma}
\begin{proof}
    Suppose there exists a probable context $\tilde{C}$ such that $p_{\theta}(|Y|=\infty\,|\,\tilde{C}) > 0$.
    Then
    \begin{align*}
        p_{\theta}(|Y|=\infty)
        & = \mathbb{E}\left[p_{\theta}(|Y|=\infty\,|\,C)\right] \\
        & \geq p(\tilde{C}) p_{\theta}(|Y|=\infty\,|\,\tilde{C}) > 0,
    \end{align*}
    which contradicts the consistency of the model $p_{\theta}$.
\end{proof}

\begin{varlemma}{3.2}
    A recurrent language model $p_{\theta}$ is consistent if $\|h_t\|_p$ is uniformly bounded for some $p\geq1$.
\end{varlemma}
\begin{proof}
    Let $B>0$ be an upper bound such that $\|h_t\|_p < B$ for all $t$.
    Let $q$ be the conjugate of $p$ satisfying $1/p+1/q=1$.
    Then we have from H\"{o}lder's inequality,
    for all $v\in V$ and $t$,
    \begin{align*}
        u_{v}^\top h_t
        \leq \|u_v^\top h_t\|_1
        \leq \|h_t\|_p\|u_v\|_q 
          < B u^+,
    \end{align*}
    where $u^+ = \max_{v\in V} \|u_v\|_q$.
    Note that
    \begin{align*}
        \log \sum_{v\in V} e^{u_{v}^\top h_t + c_{v}}
        & \leq \log \left ( \max_{v\in V} e^{u_{v}^\top h_t + c_{v}} \times |V|  \right) \\
        & \leq \max_{v\in V} \{u_{v}^\top h_t + c_{v}\} + \log |V| \\
        & < Bu^+ + c^+ + \log |V|,
    \end{align*}
    where $c^+=\max_{v\in V} c_v$.
    For a given $y_{<t}$ and context $C$,
    \begin{align*}
        & \log p_{\theta}(\eos|\,y_{<t},C) \\
         = & (u_{\eos}^\top h_t + c_{\eos}) - \log\sum_{v \in V} e^{u_{v}^\top h_t + c_{v}} \\
         > & (-Bu^+ + c_{\eos}) - (Bu^+ + c^+ + \log |V|) 
         >  -\infty,
    \end{align*}
    and it follows that $p_{\theta}(\eos |\,y_{<t},C) > \xi > 0$ for some strictly positive constant $\xi$.
    Then
    \begin{align*}
        p_{\theta}(|Y|=\infty)
        & = \lim_{t\to\infty} p_{\theta}(|Y|>t) \\
        & = \lim_{t\to\infty} \mathbb{E}\left[p_{\theta}(|Y|>t\,|\,C) \right] \\
        & = \mathbb{E}\left[\lim_{t\to\infty} p_{\theta}(|Y|>t\,|\,C) \right] \\
        & \leq \mathbb{E}\left[\lim_{t\to\infty} (1 - \xi)^t\right] = 0,
    \end{align*}
    and hence $p_{\theta}$ is consistent.
\end{proof}

\begin{varlemma}{3.3}
A consistent decoding algorithm with respect to a consistent recurrent language model decodes only probable sequences. That is, if $q_{\mathcal{F}}(Y\,|\,C)>0$, then $p_{\theta}(Y\,|\,C)>0$ for any probable context $C$.
\end{varlemma}
\begin{proof}
Suppose there exists a decoded sequence $\tilde{Y}$ by $\mathcal{F}$ and probable context $\tilde{C}$ such that
$q_{\mathcal{F}}(\tilde{Y}\,|\,\tilde{C})>0$ but $p_{\theta}(\tilde{Y}\,|\,\tilde{C})=0$.
By Remark 2.1,
the sequence $\tilde{Y}$ is of infinite length and thus $q_{\mathcal{F}}(|Y|=\infty\,|\,\tilde{C})\geq q_{\mathcal{F}}(\tilde{Y}\,|\,\tilde{C})>0$, 
which contradicts the consistency of $q_{\mathcal{F}}$ by Lemma 3.1.
\end{proof}

\section{Proofs for Section 4}
\label{apx:sec-proofs-sec4}

\begin{vartheorem}{4.1}
Suppose a recurrent LM $p_{\theta}$ has uniformly bounded $\|h_t\|_p$ for some $p\geq 1$.
If a decoding algorithm $\mathcal{F}$ satisfies $q_{\mathcal{F}}(\eos|\,y_{<t}, C) \geq p_{\theta}(\eos|\,y_{<t}, C)$ for every prefix $y_{<t}$ and context $C$,
then the decoding algorithm $\mathcal{F}$ is consistent with respect to the model $p_{\theta}$.
\end{vartheorem}
\begin{proof}
By Lemma \ref{lemma:bddrlm} 
the model $p_{\theta}$ is consistent and $p_{\theta}(\eos|\,y_{<t}, C)>\xi$ for some positive value $\xi$.
Thus, $q_{\mathcal{F}}(\eos|\,y_{<t},C) \geq p_{\theta}(\eos|\,y_{<t},C) > \xi$.
For $t\geq 1$,
\begin{align*} 
    \label{eqn:qp-ineq}
    & q_{\mathcal{F}}(|Y|>t\,|\,C) \\
    & = q_{\mathcal{F}}(y_1\neq\eos,\cdots,y_t\neq\eos|\,C) \\ 
    & \leq (1 - \xi)^t.
\end{align*}
Taking the limit $t\to\infty$ and expectation over $C$, we have
\begin{align*}
    q_{\mathcal{F}}(|Y|=\infty)
    & = \mathbb{E}_C\left[\lim_{t\to\infty} q_{\mathcal{F}}(|Y|>t\,|\,C)\right] \\
    & \leq \lim_{t\to\infty} (1 - \xi)^t = 0,
\end{align*}
from which the decoding algorithm is consistent.
\end{proof}

\begin{vartheorem}{4.2}
   Greedy decoding is consistent with respect to any self-terminating recurrent LM.
\end{vartheorem}
\begin{proof}
Let $p_{t}^{\eos}$ denote $p_{\theta}(\eos|\,y_{<t}, C)$ and $a_{t}^{\eos}$ denote $u_{\eos}^\top h_t + c_{\eos}$.
By Definition~\ref{def:strlm} we have
\begin{align*}
    p_{t}^{\eos} 
    & = 1 - \sigma(a_{t}^{\eos})(1-p_{t-1}^{\eos}) \\
    & = 1 - \prod_{t'=0}^t \sigma(a^{\eos}_{t'})
      \geq 1 - (1-\epsilon)^{t+1}.
\end{align*}
Take $B=-\log 2 / \log(1-\epsilon)$.
We then have $p_{t}^{\eos}>1/2$ for all $t > B$, which implies that $\eos$ is always the most probable token after time step $B$.
Hence, the sequence length is less than $B$ with probability 1.
\end{proof}

\begin{vartheorem}{4.3}
    Beam search with width $k$, $\mathcal{F}_{\text{beam}-k}$, is consistent with respect to any STRLM.
\end{vartheorem}
\begin{proof}
    Let $S(\rho)$ be the size-$k$ set of sequences kept by $\mathcal{F}_{\text{beam}-k}$ that start with a prefix $\rho$.
    
    Take $B=-\log 2 / \log(1-\epsilon)$ as in the proof of Theorem~4.2.
    Suppose that there exists at least one prefix $\hat{\rho} \in P^{\text{top}}_{B}$ which does not end with \eos.
    
    We first want to show that $\hat{\rho}$ induces at most $k$ more steps in beam search with width $k$, that is, $Y \in S(\hat{\rho})$ implies $|Y| \leq B+k$.
    
    We know from the proof of Theorem~4.2 that 
    an STRLM $p_{\theta}$ satisfies: for any context $C$ and $v\in V \setminus \{\eos\},$
    \begin{align*}
        p_{\theta}(\eos |\,\hat{\rho}, C) 
        > p_{\theta}(v \,|\,\hat{\rho}, C).
    \end{align*}
    For any subsequence $y=(y_1,\ldots,y_l)$ with $y_1\neq\eos$,
    \begin{align*}
        p_{\theta}(\hat{\rho}\circ y\,|\,\hat{\rho}, C)
        & = \prod_{i=1}^l p_{\theta}(y_{i}\,|\,\hat{\rho}\circ y_{<i}, C) \\
        & \leq p_{\theta}(y_{1}\,|\,\hat{\rho}, C) \\
        & < p_{\theta}(\eos |\, \hat{\rho}, C).
    \end{align*}
    Thus, $\hat{\rho} \circ \eos$ is the most probable sequence among sequences starting with the prefix $\hat{\rho}$, and it follows that $\hat{\rho} \circ \eos \in S(\hat{\rho})$.
    
    Thus, in $S(\hat{\rho})$, there are $(k-1)$ sequences starting with $\hat{\rho} \circ v$ for $v \in V\setminus\{\eos\}$.
    By the same argument, at each step at least one sequence ending with \eos\ is added to $S(\hat{\rho})$, and therefore at time step $(B+k)$, $k$ sequences ending with \eos\ are in $S(\hat{\rho})$.
    
    Note that the result set $S$ by $\mathcal{F}_{\text{beam}-k}$ (Definition~2.11) satisfies
    \begin{align*}
        S \subseteq \bigcup_{\rho \in P^{\text{top}}_B} S(\rho).
    \end{align*}
    Since each $\rho \in P^{\text{top}}_{B}$ induces sequences of length at most $B+k$, we have
    \begin{align*}
        p_{\theta}(|Y|>B+k\,|\,C) = 0.
    \end{align*}
    Taking the expectation over $C$ yields the consistency of the model $p_{\theta}$.
\end{proof}

\newpage

\section{Additional Results and Experiment Details}
\label{apx:sec-additional}
\paragraph{Training.}
Each model is trained on a single Nvidia P40 GPU for up to 100 epochs, stopping when validation perplexity does not decrease for 10 consecutive epochs.

\paragraph{Hyper-parameters.} Tables \ref{tbl:grid},\ref{tbl:grid-bpe} show the grid search specifications. All models were 2 layers and were trained with the Adam optimizer.

\paragraph{Model perplexities.} Tables \ref{tbl:rnn-results-ppl-word}, \ref{tbl:rnn-results-ppl-bpe} shows train and test perplexities for the $\tanh$-RNN and LSTM-RNN models using word and BPE tokenization, respectively.

\paragraph{Additional example continuations.} Table \ref{tbl:st-continuations-more} shows additional greedy-decoded continuations using a self-terminating LSTM-RNN and the baseline LSTM-RNN with BPE tokenization.

\paragraph{GPT-2 length distributions.}
Figure~\ref{fig:gpt2_lentgth} shows the length distributions of ground-truth continuations, continuations from GPT-2 117M, and continuations from the self-terminating GPT-2 117M.

\begin{figure}[h]
\includegraphics[width=0.45\textwidth]{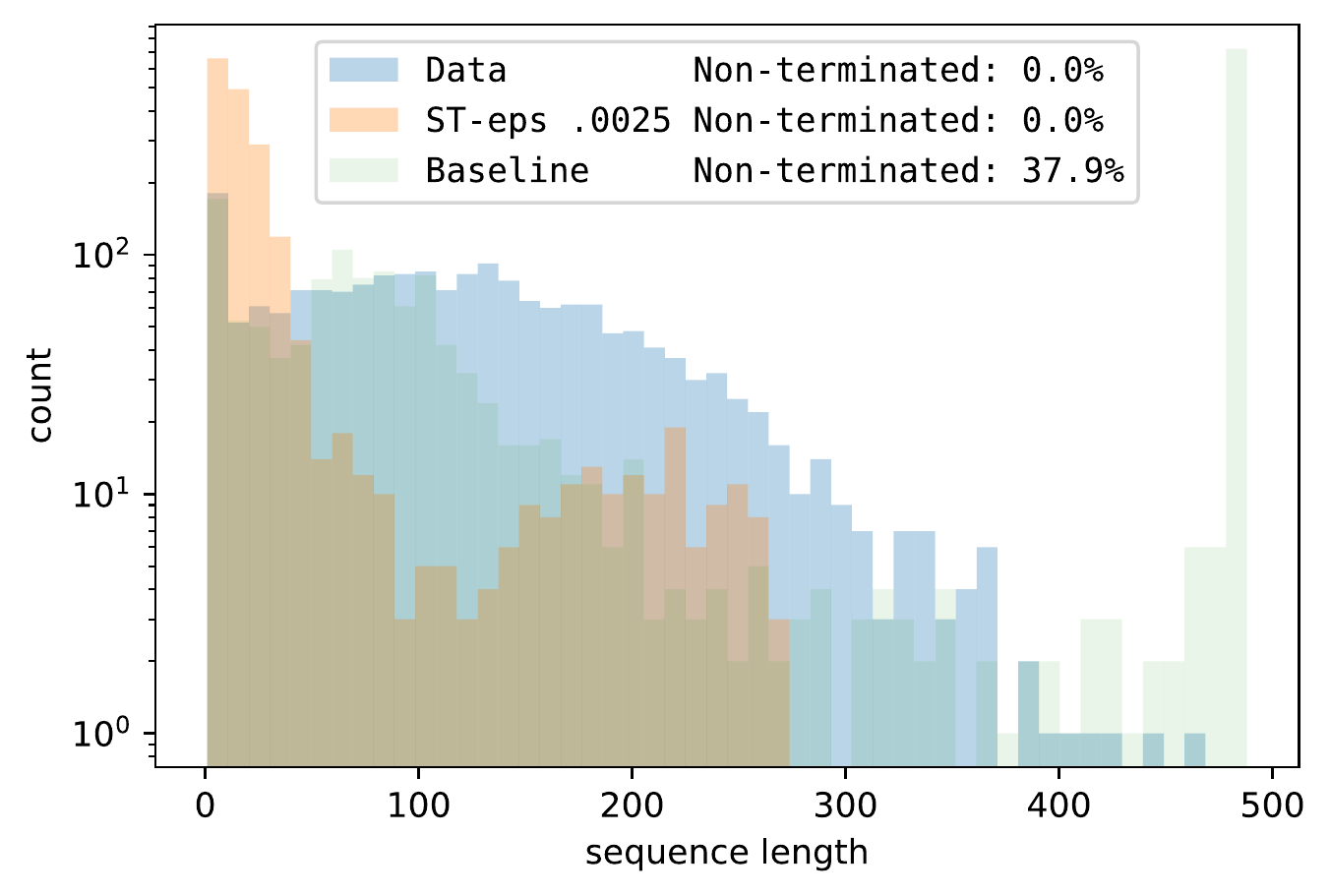}

\caption{Lengths of ground-truth and greedy-decoded continuations from the baseline GPT-2 117M and self-terminating GPT-2 117M  models ($\epsilon=0.0025$).}
\label{fig:gpt2_lentgth}
\end{figure}

\begin{table}[t]
\centering
\small 

\smallskip
\begin{tabular}{ll}
\toprule
\textbf{Parameter} & \textbf{Values}  \\
\midrule
Hidden Size & $\{\textit{256},\textbf{512},1024\}$ \\
Dropout & $\{0.1,\textit{0.3},\textbf{0.5}\}$ \\
Embedding Weight Tying & $\{\textbf{\textit{\text{True}}},\text{False}\}$ \\
\bottomrule
\end{tabular}
\caption{\label{tbl:grid} Grid search specification. The values selected for the \textbf{LSTM-RNN} and \textit{$\tanh$-RNN} models are shown in bold and italics, respectively (word tokenization).}
\end{table}

\begin{table}[t]
\centering
\small 

\smallskip
\begin{tabular}{ll}
\toprule
\textbf{Parameter} & \textbf{Values}  \\
\midrule
Hidden Size & $\{\textit{\textbf{256}},512,1024\}$ \\
Dropout & $\{0.1,\textbf{0.3},\textit{0.5}\}$ \\
Embedding Weight Tying & $\{\text{True},\textbf{\textit{\text{False}}}\}$ \\
\bottomrule
\end{tabular}
\caption{\label{tbl:grid-bpe} Grid search specification. The values selected for the \textbf{LSTM-RNN} and \textit{$\tanh$-RNN} models are shown in bold and italics, respectively (BPE tokenization).}
\end{table}

\begin{table}[t]
\centering
\small 

\begin{tabular}{llllll}
\toprule
\textbf{Type} & \# Train & \# Valid & \# Test  & $|V|$ & Avg. len \\
\midrule
\textbf{Word} & 78274 & 8464  & 9708  & 33182 & 24      \\
\textbf{BPE}  & 83344 & 8721  & 10156 & 19483 & 28     \\
\bottomrule
\end{tabular}

\caption{Wikitext2 statistics.}
\label{tbl:dataset-stat}
\end{table}

\begin{table}[t]
\centering
\small 

\begin{tabular}{lll}
\toprule
& \multicolumn{1}{l}{$\mathbf{\tanh}$\textbf{-RNN}} & \multicolumn{1}{l}{\textbf{LSTM-RNN}} \\ 
\midrule
\textbf{ancestral} & 0.00 $\pm$ 0.0 & 0.00 $\pm$ 0.0  \\
\midrule
\textbf{greedy} & 6.07 $\pm$ 5.6 & 1.03 $\pm$ 0.3\\
\textbf{beam-2} & 1.21 $\pm$ 0.3 & 0.07 $\pm$ 0.1  \\ 
\textbf{beam-4} & 0.29 $\pm$ 0.1 & 0.00 $\pm$ 0.0  \\
\midrule
\textbf{topk-2} & 0.84 $\pm$ 0.8 & 0.00 $\pm$ 0.0 \\
\textbf{topk-4} & 0.02 $\pm$ 0.0 & 0.00 $\pm$ 0.0 \\
\textbf{nucleus-0.2} & 2.49 $\pm$ 0.2 & 0.76 $\pm$ 0.3 \\
\textbf{nucleus-0.4} & 0.32 $\pm$ 0.1 & 0.22 $\pm$ 0.1 \\
\bottomrule
\end{tabular}
\caption{\label{tbl:rnn-results-decoding} Non-termination ratio ($r_L$ (\%)) of decoded sequences using ancestral sampling and incomplete decoding methods (word tokenization).}
\end{table}

\begin{table}[t]
\centering
\small 

\begin{tabular}{p{0.1cm}p{0.1cm}llllll}
\toprule
&ST&$\epsilon$&$r_L$ (\%) & \textbf{perplexity} \\
\midrule
\multirow{4}{*}{\rotatebox{90}{{$\tanh$-RNN}}}&\Checkmark&  $10^{-2}$ & 0.00 $\pm$ 0.0 & 150.07 $\pm$ 2.7 \\
&\Checkmark& $10^{-3}$ & 0.00 $\pm$ 0.0 & 138.01 $\pm$ 0.6 \\
&\Checkmark& $10^{-4}$ & 1.04 $\pm$ 0.6 & 138.67 $\pm$ 1.8 \\
& \text{\ding{55}} & -- & 6.07 $\pm$ 5.6 & 136.57 $\pm$ 1.8 \\
\midrule
\multirow{4}{*}{\rotatebox{90}{{LSTM}}}&\Checkmark& $10^{-2}$ & 0.00 $\pm$ 0.0 & 101.24 $\pm$ 0.3 \\
&\Checkmark& $10^{-3}$ & 0.00 $\pm$ 0.0 & 94.33 $\pm$ 0.6  \\
&\Checkmark& $10^{-4}$ & 0.94 $\pm$ 0.5 & 94.15 $\pm$ 0.8 \\
&\text{\ding{55}}& -- & 1.03 $\pm$ 0.3 & 91.86 $\pm$ 0.4 \\ 
\bottomrule
\end{tabular}
\caption{\label{tbl:st-results} Non-termination ratio $(r_L$ (\%)) of greedy-decoded sequences and test perplexity for self-terminating recurrent models (word tokenization).}
\end{table}

\begin{table}[ht!]
\centering
\small 

\smallskip
\begin{tabular}{lll}
\toprule
\textbf{model} & \textbf{context} & \textbf{perplexity} \\ 
\midrule
$\tanh$-RNN & train & {91.54 $\pm$ 7.9}\\
$\tanh$-RNN & test & {136.57 $\pm$ 1.8}\\
\midrule
LSTM-RNN & train & {45.80 $\pm$ 2.5}\\
LSTM-RNN & test & {91.86 $\pm$ 0.4}\\
\bottomrule
\end{tabular}
\caption{Perplexities of trained recurrent language models (word tokenization).}
\label{tbl:rnn-results-ppl-word}
\end{table}

\begin{table}[ht!]
\centering
\small 

\smallskip
\begin{tabular}{lll}
\toprule
\textbf{model} & \textbf{context} & \textbf{perplexity} \\ 
\midrule
$\tanh$-RNN & train & {61.20 $\pm$ 1.2}\\
$\tanh$-RNN & test & {186.44 $\pm$ 1.4}\\
\midrule
LSTM-RNN & train & {72.72 $\pm$ 2.4}\\
LSTM-RNN & test & {178.39 $\pm$ 1.2}\\
\bottomrule
\end{tabular}
\caption{ Perplexities of trained recurrent language models (BPE tokenization).}
\label{tbl:rnn-results-ppl-bpe}
\end{table}

\begin{table*}[h!]
\vspace{-4mm}
\centering
\resizebox{\linewidth}{!}{
\begin{tabular}{lp{1.05\linewidth}}
\toprule
\textbf{Prefix} & \textit{Payne was quoted as saying : `` With the album}\\
\textbf{nucleus} & 's '' album , `` The Predious '' , `` The One With the Wind '' , `` I 've Ever 't Have You 's My Way '' , `` I 've Ever It 's Johnny '' , `` The Predes '' , `` `` Always '' , `` The Predatory Was '' , `` The Dream '' , `` The Baste '' , '' `` Always Boy '' , `` My Drum '' , `` The Simpsons '' , `` `` Always Man '', `` The `` Sweet Night '' ,  $\ldots$ \\
\textbf{c-nucleus} & 's " album , `` The Predious " , `` The One With the Wind " , `` I 've Ever 't Have You 's My Way " \eos \\
\midrule
\textbf{Prefix} & \textit{In a 2006 interview , fellow actor Ben Whish}\\
\textbf{nucleus} & 's father , a young actor , and a romantic relationship with the show , `` The One Where the The Simpsons " , `` The Pape " , `` The Next Generation " , `` The Sixth Extinction " , `` We 't You Wanna Stay " , `` The Dream " , `` The Predator " , `` The Collection " , `` The Big Lear " , `` The Predor " , `` The Predation " , `` My Blue " , `` The Simpsons " , `` The Sixth Extinction " , `` My Love " , `` The Rise of the Year " , `` The Simpsons " , `` The Predator " , `` My Dream " ,  $\ldots$ \\
\textbf{c-nucleus} & was the first time in the film , and was published in the same episode of the season . \eos \\
\midrule
\textbf{Prefix} & \textit{Most of what is known of Du Fu 's}\\
\textbf{Baseline} &  `` the " , the '' `` great " , the " `` " , `` the most important " , `` the most important " , `` Ode to the Nightingale " , `` Ode to the Nightingale " , `` Ode to the Nightingale " , `` Ode to the Nightingale " , `` Ode to the Nightingale " , `` Ode to the Nightingale " , `` Ode to the Nightingale " , `` Ode to the Nightingale " , $\ldots$  \\
\textbf{STRLM} & Coty , was a `` one of the most important " of the American science fiction . \eos\\
\midrule
\textbf{Prefix} & \textit{He was relieved by Yan Wu , a friend and}\\
\textbf{Baseline} &  the first wife of the Order of the Order of the Order of the Order of the Order of the Republic of the Republic of the Republic of the Republic of the Republic of the Republic of the Republic of the Republic of the Republic of the Republic of the Republic of the Republic of the Republic $\ldots$  \\
\textbf{STRLM} & the wife of the Royal Navy . \eos\ \\
\bottomrule
\end{tabular}
}
\caption{More continuations with consistent nucleus sampling ($\mu = 0.2$) and self-terminating LSTM ($\epsilon=10^{-3}$) with BPE tokenization.}
\label{tbl:st-continuations-more}
\end{table*}

\end{document}